\documentclass[journal]{IEEEtran}

\usepackage{amsmath,amsfonts}
\usepackage{algorithmic}
\usepackage{algorithm}
\usepackage{array}
\usepackage[caption=false,font=normalsize,labelfont=sf,textfont=sf]{subfig}
\usepackage{textcomp}
\usepackage{booktabs}
\usepackage{stfloats}
\usepackage{url}
\usepackage[export]{adjustbox}
\usepackage{balance}
\usepackage{multirow}
\usepackage[table,xcdraw]{xcolor}
\usepackage{multirow}
\usepackage{graphicx}
\usepackage{xcolor}
\usepackage{verbatim}
\usepackage{graphicx}
\usepackage{amsthm}     
\usepackage{amsmath}    
\usepackage{amssymb}    

\newtheorem{lemma}{Lemma}
\newtheorem{hypothesis}{Hypothesis}
\usepackage{fix-cm}
\usepackage{hyperref}
\usepackage{booktabs}
\usepackage{graphicx}
\usepackage{multirow}
\usepackage{amsmath}
\usepackage{xcolor} 
\usepackage{amsthm}

\begin{document}
%
%
%
\title{GANet: A Lightweight Group Attention Network for Plant Disease Identification on Embedded Systems}

\author{\IEEEauthorblockN{Anubhab Tripathi\IEEEauthorrefmark{2},
Bibek Das\IEEEauthorrefmark{2},
Asfak Ali\IEEEauthorrefmark{1}\IEEEauthorrefmark{2}, and 
Sheli Shinha Chaudhuri\IEEEauthorrefmark{2}}%
\thanks{\IEEEauthorrefmark{2}Department of Electronics and Telecommunication Engineering, Jadavpur University, Kolkata, India (e-mail: tripathianubhab@gmail.com, bibekdas.etce@gmail.com, asfakali.etce@gmail.com, shelism@rediffmail.com).}%
\thanks{\IEEEauthorrefmark{1}Corresponding author: Asfak Ali (e-mail: asfakali.etce@gmail.com).}%
}

\markboth{Journal of \LaTeX\ Class Files,~Vol.~14, No.~8, August~2015}%
{Shell \MakeLowercase{\textit{et al.}}: Bare Demo of IEEEtran.cls for IEEE Transactions on Magnetics Journals}
%
%
\IEEEtitleabstractindextext{%
\begin{abstract}
Global warming and climate change are major catalysts for plant diseases, significantly affecting food production worldwide. Addressing this issue requires continuous monitoring and early detection of plant diseases. While recent advancements in deep learning offer potential solutions, most models demand high processing power and struggle to perform efficiently on low-end electronic hardware. Furthermore, many existing works focus on specific crop diseases, limiting their broader applicability. This paper presents a lightweight Group Equivariance Attention Network (GANet) designed for the prediction and monitoring of 38 types of plant diseases. GANet introduces a novel Group Attention block and Group-Wise Feature Fusion block, making it both sparse and efficient. With only 0.001M parameters, the model achieves an impressive 99.99\% accuracy on the PlantVillage dataset. Additionally, GANet is optimized for deployment on edge devices such as the Raspberry Pi 4b and Jetson Nano, where it reaches 5 FPS and 11 FPS, respectively, demonstrating its potential for real-time plant disease monitoring on low-power devices.
\end{abstract}

\begin{IEEEkeywords}
IEEE, IEEEtran, IEEE Transactions on Magnetics, journal, \LaTeX, magnetics, paper, template.
\end{IEEEkeywords}}

\maketitle

\IEEEdisplaynontitleabstractindextext

%
\IEEEpeerreviewmaketitle

\section{Introduction}\label{sec1}

Plant diseases pose a significant threat to agricultural production, necessitating timely identification to mitigate their impact \cite{faithpraise2013automatic,nagaraju2020systematic}. Accurate diagnosis is crucial for ensuring high-quality agricultural products. This research aims to simplify the process of providing timely information about plant conditions, addressing the limitations of conventional disease monitoring techniques that prevent early intervention by local farmers \cite{zhang2020review}. The demand for efficient, quick, and cost-effective methods to automatically detect crop diseases is growing, driven by advancements in digital cameras, computational power, and the integration of Machine Learning and Digital Image Processing \cite{bera2019survey}. Most plant diseases are initially detected in the leaves, making image classification algorithms a viable solution for automated identification \cite{ebrahimi2017vision}. Computer Vision offers a non-contact, inexpensive approach to autonomous crop monitoring \cite{gayathri2019image}. Over the past fifteen years, primary classification algorithms such as Artificial Neural Networks (ANNs) \cite{sheikhan2012improved}, k-Nearest Neighbour (KNN) \cite{guettari2016blind}, and Support Vector Machines (SVMs) \cite{deepa2017steganalysis} have been widely used. Recently, Deep Learning (DL) approaches, particularly Convolutional Neural Networks (CNNs), have gained popularity due to their effectiveness in solving specific problems \cite{barbedo2019plant}. CNNs are widely used for image recognition tasks, both large \cite{huang2018deep} and small \cite{tang2018median}, due to their remarkable performance across various applications \cite{hu2019driving,kamilaris2018deep}. However, these models often have complex layers with millions of parameters, which retain information rather than understanding image transformations like translation, rotation, and symmetry. In order to exploit larger groups of symmetries in images and create smaller, more general networks, the idea of equivariance is essential. In order to identify plant diseases, this study suggests a lightweight Group Equivariance attention network with fewer parameters, lower computational costs, and simpler implementation.

The remainder of this study is described as follows:
\autoref{sec2} revised the work done in the past on plant disease recognition.In \autoref{sec3}, \autoref{sec4} and \autoref{sec5} generally discuss the proposed methodology, result analysis, and conclusion of this research.

\begin{figure*}[!t]
\centerline{\includegraphics[width=.95\textwidth]{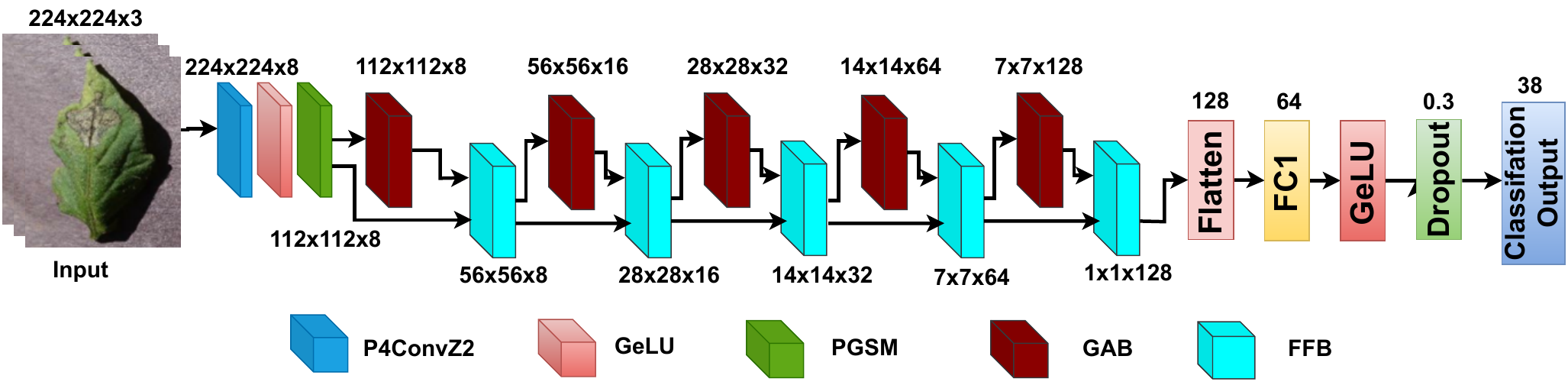}}
\caption{Overall Block Diagram of the Leaf Disease Multi-class Classification Model.}
\label{fig: Fianl}
\end{figure*}

\section{Literature Survey}\label{sec2}

Numerous studies have evaluated the performance of various deep learning architectures on datasets containing leaf images of varying proportions. These studies range from those achieving impressive results without explicit feature engineering \cite{mohanty2016using} to those decisively stating that ResNet-50 offers the best architecture \cite{9753846}. Diverse ideas have emerged in the quest to identify the most suitable model for plant disease detection. One notable study \cite{khan2023plant} focused on developing a plant disease detection model tailored for edge computing devices. This work utilized a dataset to train and optimize a MobileNetV3-small model \cite{howard2019searching}, ensuring efficient deployment on resource-constrained devices. A comprehensive literature review \cite{pandey2024survey} summarized 12 algorithms employing color, texture, shape, and hybrid approaches to classify plant leaf diseases from digital images. The accuracy rates of these algorithms varied from 80\% to 99\%, depending on the specific algorithm and dataset used.
Another systematic literature review \cite{ramanjot2023plant} discussed various aspects of plant disease detection, including main sources for collecting plant data, pre-processing and data augmentation techniques, feature extraction methods, attributes extracted, automated disease identification systems, image quality improvement methods, techniques to reduce overfitting, plant species, and disease classes studied. This review also highlighted the accuracy of existing approaches. These studies collectively underscore the significance of the PlantVillage dataset in advancing both research and practical applications in the field of automated plant disease diagnosis.


\begin{figure}[!t]
\centerline{\includegraphics[width=.37\textwidth]{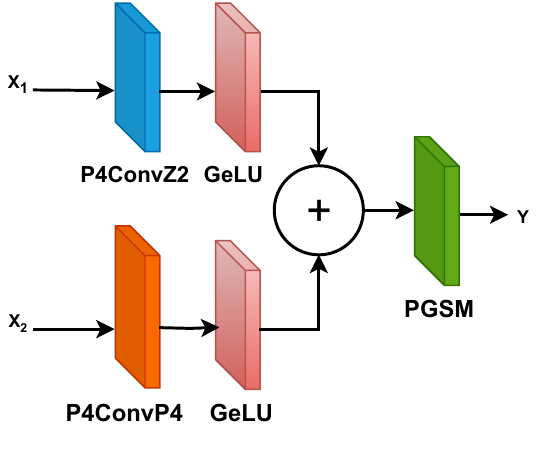}}
\caption{Block Diagram of Feature Fusion Block}
\label{fig: fusion}
\end{figure}
\section{Proposed Methodology}\label{sec3}
\subsection{Foundation of Group Convolution}
The fundamental concept of Group-equivariant Convolutional Neural Networks (G-CNNs), introduced by Cohen et al. in \cite{cohen2016group}, redefines the convolutional operation by embedding group theory principles into the design of convolutional layers. Traditional convolutional layers exploit translational invariance, which assumes that shifting an object within an image should not alter its representation. However, G-CNNs generalize this principle by introducing \textit{group-equivariant convolutions}\cite{group1}, which extend the scope of symmetry handling to include transformations such as rotations and reflections.

Mathematically, a group \( G \) is a set of transformations under which the data exhibits symmetries. For example, the dihedral group \( D_4 \) represents symmetries of a square, including rotations (by multiples of \(90^\circ\)) and reflections. In the context of convolutional neural networks, a G-equivariant convolution satisfies the property:
\begin{equation}
\Phi(g \cdot x) = g \cdot \Phi(x), \quad \forall g \in G,
\end{equation}
where \( x \) represents the input, \( g \cdot x \) denotes the transformation of \( x \) by \( g \), and \( \Phi \) is the convolution operation. This property ensures that applying a group transformation \( g \) to the input yields a predictable transformation of the feature map.

Two notable types of group convolutions discussed in \cite{cohen2016group} are as follows:

\subsubsection{\textbf{p4convp4 Convolution}}
This convolutional layer is designed to be equivariant under the dihedral group \(D_4\), which includes symmetries of a square such as rotations by multiples of \(90^\circ\) and reflections. The dihedral group \(D_4\) can be formally defined by the following presentation:
\begin{equation}
D_4 = \langle a, b \mid a^4 = b^2 = e, \, ab = ba^{-1} \rangle,
\end{equation}
where:
\begin{itemize}
    \item \(a\) represents a counterclockwise rotation by \(90^\circ\),
    \item \(b\) represents a reflection,
    \item \(e\) is the identity element.
\end{itemize}

In the \(p4convp4\) layer, filters are shared across all transformations within \(D_4\), enabling the network to effectively process features invariant to these transformations. For an input feature map \(x\), the output feature map \(y\) is computed as:
\begin{equation}
y[g] = \sum_{h \in G} k[g^{-1}h] \cdot x[h], \quad \forall g \in D_4,
\end{equation}
where:
\begin{itemize}
    \item \(k\) denotes the filter kernel,
    \item \(g^{-1}h\) represents the transformation of the kernel by the inverse of \(g\),
    \item \(h \in G\) denotes the group elements applied to the input \(x\).
\end{itemize}
This formulation ensures that a single filter adapts to all group transformations, significantly reducing the number of parameters and improving generalization.

\subsubsection{\textbf{p4convz2 Convolution}}
The \(p4convz2\) convolution extends the capabilities of \(p4convp4\) by incorporating an additional symmetry group \(Z_2\), which represents binary reflections (e.g., flipping). The combined group, \(Z_2 \times D_4\), enriches the network's ability to handle more complex symmetries. The convolution operation for \(p4convz2\) is mathematically defined as:
\begin{equation}
y[g, z] = \sum_{h \in G} k[g^{-1}h, z^{-1}] \cdot x[h, z], \quad \forall g \in D_4, \, z \in Z_2,
\end{equation}
where:
\begin{itemize}
    \item \(z \in Z_2\) represents the additional symmetry dimension,
    \item \(k[g^{-1}h, z^{-1}]\) denotes the transformed filter kernel considering both rotation and reflection symmetries,
    \item \(x[h, z]\) is the transformed input feature map under the combined group actions.
\end{itemize}

By leveraging this\(Z_2 \times D_4\), the \(p4convz2\) layer provides richer symmetry handling, further enhancing the model's ability to generalize across varied orientations and reflections. \newline

Group-equivariant convolutions provide a substantial improvement in both parameter efficiency and generalization capabilities of neural networks. In contrast to traditional convolutional layers, which require distinct filters to learn features for each transformation, G-convolutions enable a single filter to capture patterns invariant to all transformations within the group \(G\). This approach not only reduces the number of parameters but also mitigates the risk of overfitting, particularly in datasets with limited training samples. Additionally, by inherently handling transformations such as rotations and reflections, G-convolutions eliminate the need for extensive data augmentation. These attributes make G-CNNs a powerful and efficient solution for tasks demanding invariance to complex symmetries, such as the classification of plant diseases, where variability in orientation and structural patterns is prevalent.

\subsection{Proposed Model}

We propose a novel model for plant disease classification that integrates Group-equivariant Convolutional Neural Networks (G-CNNs) with a Feature Fusion Block (FFB) and an Attention Block (AB), as illustrated in \autoref{fig: Fianl}. This architecture leverages \( P4ConvZ2 \) and \( P4ConvP4 \) convolutions to capture geometric symmetries such as rotations and reflections, ensuring robustness against transformations common in plant disease imagery. These group-equivariant convolutions process feature maps \( x \) using shared filters across the transformation group \( G = D_4 \times Z_2 \), where \( D_4 \) denotes the dihedral group of eight symmetries, and \( Z_2 \) accounts for flipping invariance. The group-convolution operation is defined mathematically as:

\[
(y \ast \psi)(g) = \sum_{h \in G} x(h) \psi(g^{-1} h),
\]

where \( x(h) \) is the input feature map, \( \psi \) is the shared filter, and \( y \) is the output feature map. This shared-filter mechanism reduces the number of parameters, thereby enhancing efficiency and mitigating overfitting.
\begin{figure}[!t]
\centerline{\rotatebox{-90}{\includegraphics[width=0.25\textwidth]{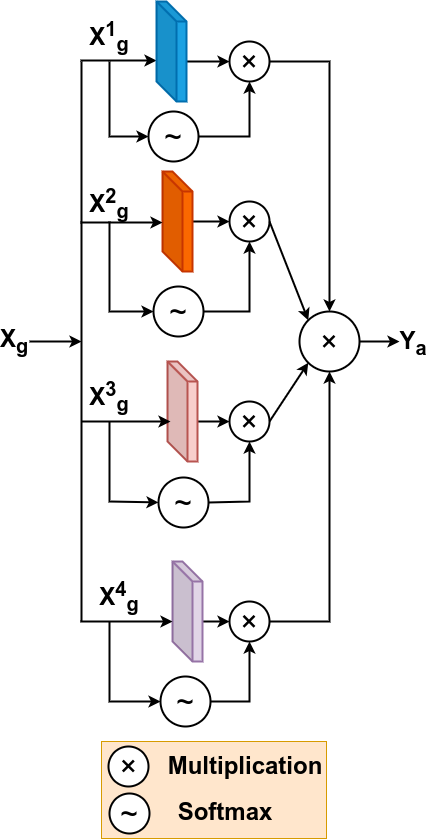}}}
\caption{Block Diagram of Attention Mechanism system.}
\label{fig: Attension}
\end{figure}
The proposed architecture employs group-equivariant convolutional layers, embedded with Plane Group Spatial Max Pooling (PGSM) to downsample feature maps. These layers are followed by a fully connected layer with 64 neurons, a dropout layer, and a softmax output layer that classifies images into one of 38 plant disease categories. The softmax function, defined as:

\[
\hat{y}_i = \frac{\exp(z_i)}{\sum_{j=1}^C \exp(z_j)},
\]

maps the logits \( z_i \) into normalized class probabilities, where \( C \) is the total number of classes.

\subsubsection{Feature Fusion Block (FFB)}

The FFB, shown in \autoref{fig: fusion}, enhances multi-scale feature representation by aggregating outputs from \( P4ConvZ2 \) and \( P4ConvP4 \) layers. The fusion operation combines the feature maps \( F_{Z_2} \) and \( F_{P_4} \) through element-wise summation:

\[
F_{\text{fusion}} = F_{Z_2} + F_{P_4},
\]

where \( F_{\text{fusion}} \) represents the aggregated feature map. This fusion mechanism integrates information from multiple transformations, enriching the feature space and enabling the model to capture complex patterns associated with plant disease images. This operation is particularly effective in handling visual variations and subtle differences inherent in plant disease datasets.

\begin{lemma}\textbf{Equivariance Preservation under Group Convolution Fusion: }
Let \( G \) be a finite group, and let \( \Phi_1, \Phi_2 : \mathcal{X} \to \mathcal{Y} \) be two group-equivariant mappings with respect to the group action \( g \cdot x \) for all \( g \in G \) and \( x \in \mathcal{X} \). That is,
\[
\Phi_1(g \cdot x) = g \cdot \Phi_1(x), \quad \Phi_2(g \cdot x) = g \cdot \Phi_2(x), \quad \forall g \in G.
\]
Then, the fused feature map \( \Phi_{\text{fusion}} : \mathcal{X} \to \mathcal{Y} \) defined by
\[
\Phi_{\text{fusion}}(x) = \Phi_1(x) + \Phi_2(x)
\]
is also group-equivariant with respect to \( G \). That is,
\[
\Phi_{\text{fusion}}(g \cdot x) = g \cdot \Phi_{\text{fusion}}(x), \quad \forall g \in G.
\]
\end{lemma}

\begin{proof}[Proof Sketch]
By linearity of addition and the equivariance properties of \( \Phi_1 \) and \( \Phi_2 \), we have:
\[
\Phi_{\text{fusion}}(g \cdot x) = \Phi_1(g \cdot x) + \Phi_2(g \cdot x)\] \[= g \cdot \Phi_1(x) + g \cdot \Phi_2(x)\] \[= g \cdot (\Phi_1(x) + \Phi_2(x))\] \[= g \cdot \Phi_{\text{fusion}}(x).
\]
Hence, \( \Phi_{\text{fusion}} \) is equivariant under the action of \( G \).
\end{proof}

\subsubsection{Group-Attention Block (GAB)}

The attention mechanism, illustrated in \autoref{fig: Attension}, selectively emphasizes the most discriminative regions in the input feature map \( X_g \). The block splits \( X_g \) into four components \( \{X_g^1, X_g^2, X_g^3, X_g^4\} \), and applies a self-attention mechanism to enhance features for each component. Specifically, a softmax-weighted feature refinement is applied as follows:

\[
\tilde{X}_{g_i} = X_{g_i} \cdot \text{softmax}(X_{g_i}),
\]

where:

\[
\text{softmax}(X_{g_i}) = \frac{\exp(X_{g_i})}{\sum_j \exp(X_{g_j})}.
\]

The self-attention mechanism computes attention weights that highlight the importance of specific regions within each component, enabling the block to focus on the most relevant features while suppressing less informative ones.

The final output of the self-attention block is obtained through element-wise multiplication of the refined components:

\[
Y_a = \tilde{X}_{g_1} \cdot \tilde{X}_{g_2} \cdot \tilde{X}_{g_3} \cdot \tilde{X}_{g_4}.
\]

This group-attention block effectively enhances critical regions in the feature map while reducing noise, thereby improving the model’s classification accuracy. By focusing on the most informative areas, the GAB complements the invariant features captured by the G-CNN layers, providing a robust mechanism for refining feature representations.

By integrating G-CNNs for symmetry handling, FFB for feature augmentation, and GAB for attention-based refinement, the proposed model demonstrates superior robustness and accuracy in plant disease classification, particularly in complex agricultural datasets.
\begin{hypothesis}\textbf{Attention-Aided Equivariance Boosting Hypothesis: }
Let \( \Phi : \mathcal{X} \to \mathcal{Y} \) be a group-equivariant feature extractor with respect to a group \( G \), and let \( \mathcal{A} : \mathcal{Y} \to \mathcal{Y} \) denote a self-attention mechanism applied over the feature space. Assume \( \mathcal{A} \) operates on localized patches or group-aligned subspaces \( \{Y_g\}_{g \in G} \subset \mathcal{Y} \). Then, the composite mapping \( \mathcal{F} = \mathcal{A} \circ \Phi \) exhibits enhanced selectivity for group-invariant salient features, i.e.,
\[
\mathcal{F}(g \cdot x) \approx g \cdot \mathcal{F}(x), \quad \forall g \in G, \, x \in \mathcal{X},
\]
under the assumption that \( \mathcal{A} \) preserves the structure induced by \( G \). Consequently, attention improves the equivariant model’s robustness to transformation-aligned feature localization.
\end{hypothesis}

\subsection{Dataset}
The PlantVillage dataset \cite{hughes2015open} is a widely utilized resource in the field of plant disease detection and classification through machine learning approaches. This comprehensive, publicly accessible dataset comprises over 54,000 high-resolution images of plant leaves. It includes images from 14 distinct crop species, with a total of 38 classes representing both healthy and diseased plants. The dataset covers a wide range of plant diseases, providing a diverse and balanced set of images, making it ideal for training and evaluating machine learning models. The images in this dataset have been carefully annotated, ensuring high-quality labels for both healthy and infected plant samples. Given its size and diversity, the PlantVillage dataset has become a standard benchmark for evaluating the performance of various image classification models in agricultural applications.

\subsection{Training and Hyperparameter Tuning}

The plant disease classification model is trained using a meticulously designed pipeline that optimizes both performance and generalization. The training process begins with a dataset of standardized 224 × 224 resolution images, augmented through random rotations (up to 15 degrees), random horizontal flips, and normalization. These augmentations enhance the model's robustness by simulating diverse real-world conditions. The Adam optimizer, known for its adaptive learning rate properties, is employed with an initial learning rate of 0.001. To ensure efficient convergence, the learning rate is dynamically adjusted using the \texttt{ReduceLROnPlateau} scheduler, which reduces the rate when validation loss plateaus. The CrossEntropyLoss function, commonly used for multi-class classification, is utilized to minimize the discrepancy between predicted and true labels. The model is trained for 100 epochs with a batch size of 128, allowing efficient computation and thorough training for convergence. A dropout rate of 0.5 is applied in the fully connected layers to mitigate overfitting and improve generalization. Training and validation losses are monitored throughout, with metrics such as accuracy and validation accuracy used to evaluate model performance across various disease categories. Model checkpoints are saved based on validation loss improvements, and early stopping is implemented to halt training if validation loss stagnates, avoiding unnecessary computation. This comprehensive configuration ensures high accuracy and robust generalization for real-world deployment. Table \ref{tab:exp_config} summarizes the key hyperparameters used in the experiment.

\begin{table}[!t]
    \centering
    \caption{Experimental Parameter Configuration}
    \resizebox{0.4\textwidth}{!}{%
        \begin{tabular}{@{}ll@{}}
            \toprule
            \textbf{Configuration}          & \textbf{Specification}              \\ \midrule
            Optimizer                      & Adam                        \\
            Loss function                  & CrossEntropyLoss            \\
            Initial learning rate          & 0.001                     \\
            Batch size                     & 128                          \\
            Learning rate schedule         & ReduceLRonPlateau                \\
            Training epochs                & 100                         \\
            Image resolution               & 224 × 224                   \\ 
            \bottomrule
        \end{tabular}%
    }
    \label{tab:exp_config}
\end{table}


\section{Result Analysis}\label{sec4}

\subsection{Experimental Result}
On the testing dataset, the model's remarkable 99.75\% accuracy showed that it could correctly classify the great majority of the samples. In addition to its high accuracy, the model also reported a loss of 0.009, indicating that its errors were minimal. The model's ability to generalize well to novel, unseen samples is demonstrated by its high accuracy and low loss. These findings suggest that the model can be confidently implemented and is a good fit for the current task. The result obtained using our model is now compared against some State-of-the-Art models in the table given below:

\begin{table*}[!t]
    \centering
    \caption{Comparison of results between Our and other studies on the PlantVillage dataset.}
    \resizebox{0.9\textwidth}{!}{%
        \begin{tabular}{lcccccccccc}
            \toprule
            \textbf{Model} & \textbf{No. of Parameters} & \textbf{Size (MB)} & \textbf{FLOP} & \textbf{Depth} & \textbf{Accuracy} & \textbf{Precision} & \textbf{Recall} & \textbf{F1-Score} & \textbf{AuC} \\
            \midrule
            VGG-19           & 143,667,240 & 549  & 1.96E+10  & 19  & 0.96  & 0.94 & 0.92 & 0.92 & 0.97 \\
            ResNet-152       & 58,772,356  & 384  & 1.13E+10  & 152 & 0.95  & 0.95 & 0.91 & 0.93 & 0.96 \\
            DenseNet-201     & 20,242,984  & 80   & 4.37E+09  & 201 & 0.94  & 0.95 & 0.92 & 0.93 & 0.94 \\
            MobileNetV2      & 3,538,984   & 14   & 3.14E+08  & 53  & 0.90  & 0.94 & 0.90 & 0.88 & 0.89 \\
            ResNet-50        & 25,636,712  & 89   & 1.10E+10  & 50  & 0.90  & 0.90 & 0.89 & 0.88 & 0.89 \\
            VGG-16           & 138,357,544 & 528  & 1.58E+10  & 23  & 0.96  & 0.94 & 0.92 & 0.90 & 0.93 \\
            InceptionV3      & 23,851,784  & 195  & 2.00E+09  & 159 & 0.95  & 0.94 & 0.94 & 0.94 & 0.96 \\
            Xception         & 22,910,480  & 88   & 1.56E+10  & 126 & 0.95  & 0.94 & 0.94 & 0.94 & 0.97 \\
            MobileNet        & 4,253,684   & 16   & 5.69E+11  & 88  & 0.96  & 0.98 & 0.97 & 0.97 & 0.99 \\
            Deep CNN         & 121,964     & 11   & 1.32E+08  & 88  & 0.98  & 0.98 & 0.97 & 0.97 & 0.99 \\
            \midrule
           \textbf{GANet (Ours)} & - & - & - & - & 0.99 & - & - & - & - \\
            \bottomrule
        \end{tabular}
    }
    \label{tab:plant_disease_classification}
\end{table*}

\subsection{Confusion Matrix}

\begin{figure*}[!t]
\centerline{\includegraphics[width=0.85\textwidth]{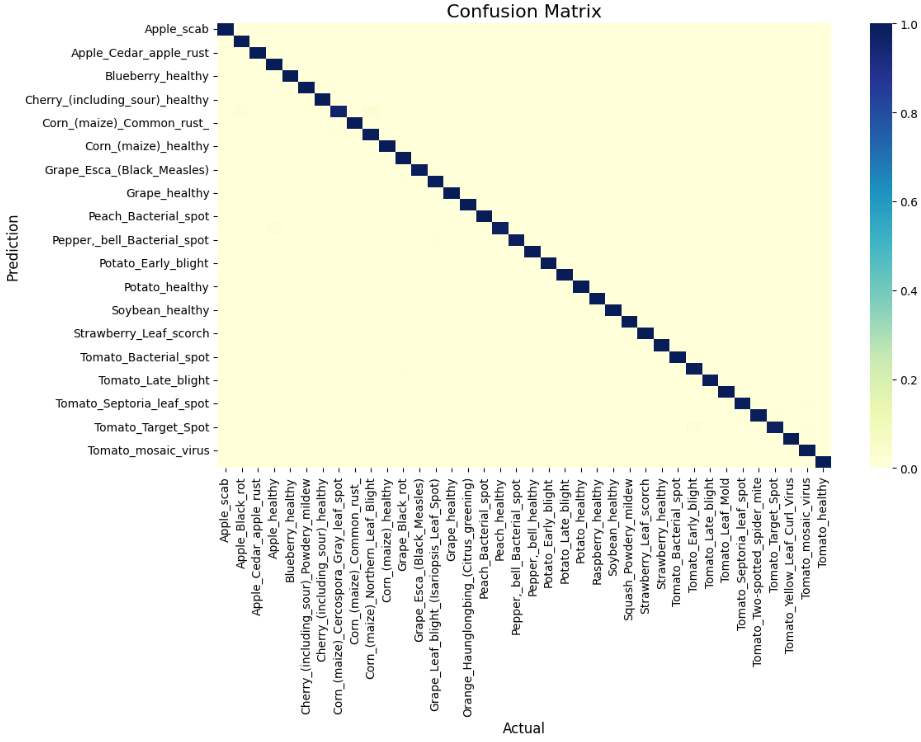}}
\caption{Multi-label Confusion Matrix.}
\label{fig: ConfusionMatrix}
\end{figure*}

Figure \ref{fig: ConfusionMatrix} illustrates the confusion matrix for the multi-class classification task, providing a detailed evaluation of the model's performance. Each row in the matrix corresponds to the actual class label, while each column represents the predicted class label. The diagonal elements reflect the correctly classified samples for each class, with darker shades indicating higher values. These diagonal entries directly measure the model’s accuracy for individual classes. The off-diagonal elements represent instances of misclassification, offering insights into potential confusion between specific classes. Elevated values in certain off-diagonal entries may indicate similarities in the feature space of the corresponding classes or limitations in the model’s ability to distinguish them effectively. To ensure fairness in performance evaluation, the confusion matrix is normalized, addressing class imbalances present in the dataset. Normalization ensures that the performance is assessed equitably across all classes, regardless of variations in sample size. This adjustment is particularly important for datasets with significant over- or under-representation of certain classes.

The following metrics further quantify the model's performance:

\begin{itemize}
    \item \textbf{Accuracy} measures the proportion of correctly classified samples and is defined as:
    \begin{equation}
    \text{Accuracy} = \frac{T_P + T_N}{T_P + T_N + F_P + F_N}
    \end{equation}
    where $T_P$, $T_N$, $F_P$, and $F_N$ represent True Positives, True Negatives, False Positives, and False Negatives, respectively. This study achieves an accuracy of 99.75\%. 
    \newline

    \item \textbf{Precision} indicates the proportion of correctly predicted positive instances out of all predicted positives:
    \begin{equation}
    \text{Precision} = \frac{T_P}{T_P + F_P}
    \end{equation}

    \item \textbf{Recall} reflects the proportion of correctly predicted positive instances out of all actual positives:
    \begin{equation}
    \text{Recall} = \frac{T_P}{T_P + F_N}
    \end{equation}

    \item \textbf{$F_1$ Score}, the harmonic mean of Precision and Recall, provides a balanced measure of these two metrics:
    \begin{equation}
    F_1 = \frac{2 \cdot \text{Precision} \cdot \text{Recall}}{\text{Precision} + \text{Recall}}
    \end{equation}
\end{itemize}

The confusion matrix, combined with these metrics, provides a comprehensive understanding of the model’s strengths and limitations. By visualizing the distribution of correct and incorrect classifications, one can identify specific areas for improvement, such as refining the training data, optimizing model parameters, or addressing systematic biases in classification.

\begin{figure}[!t]
\centerline{\includegraphics[width=270pt]{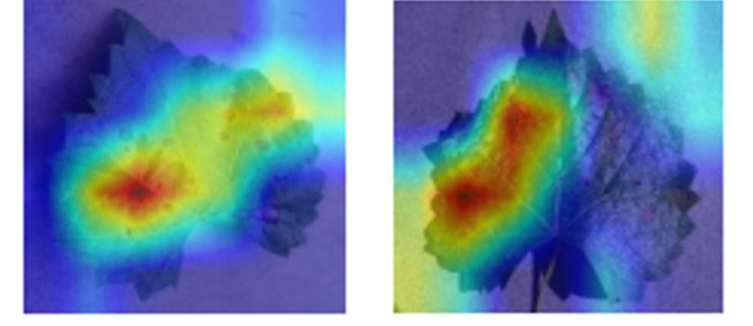}}
\caption{GradCAM Images of two leaves present in the dataset.}
\label{fig:GradCAM Images}
\end{figure}




\begin{figure*}[!t]
\centerline{\includegraphics[width=0.85\textwidth]{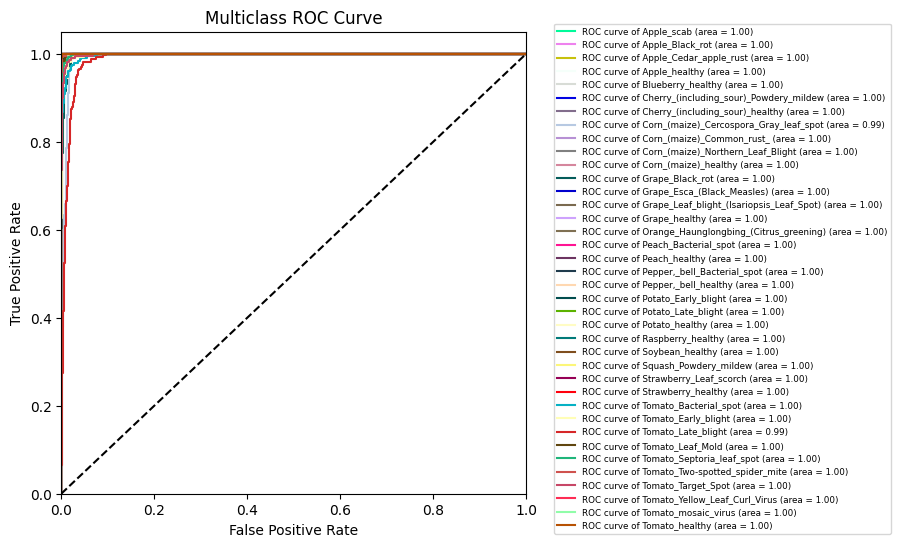}}
\caption{Multiclass ROC curves for all the classes along with their AUCs.}
\label{fig:Multiclass ROC}
\end{figure*} 

\subsection{GradCAM}

The Figure \ref{fig:GradCAM Images} illustrates the Gradient-weighted Class Activation Mapping (Grad-CAM) visualizations, which highlight the regions of the input images that the model considers most critical for its classification decisions. Grad-CAM computes a class-specific heatmap by utilizing the gradients of the classification score with respect to the feature maps in the final convolutional layer. This heatmap is overlaid on the input image to provide a visual explanation of the model's decision-making process. For example in this visualizations, warmer colors such as red and yellow indicate regions of higher importance for the model’s prediction, while cooler colors such as blue and purple represent regions with minimal impact. 

\subsection{Precision vs Recall Curves}

Figure \ref{fig:Precision vs Recall Curves} illustrates the precision-recall (PR) curves for each class in the multi-class classification task, providing a detailed assessment of the model’s ability to balance precision and recall across varying decision thresholds. The PR curve graphically represents the trade-off between precision, which measures the proportion of predicted positive cases that are correct, and recall, which measures the proportion of actual positive cases that are correctly identified. Each curve in the figure corresponds to a specific class, with the area under the precision-recall curve (PR) serving as a summary statistic that quantifies the overall performance for that class. A higher PR value indicates better precision-recall trade-offs, with an ideal curve reaching the top-right corner, reflecting both high precision and high recall. 

The figure highlights notable differences in the model's performance across classes. For instance, classes such as \textit{Apple\_healthy}, \textit{Blueberry\_healthy}, and many others achieve nearly perfect curves with PR values of 1.0, indicating exceptional classification performance with minimal false positives and false negatives. In contrast, classes such as \textit{Corn\_(maize)\_Cercospora\_Gray\_leaf\_spot} exhibit more irregular curves and lower AUPRC values, suggesting challenges in accurately distinguishing these classes, potentially due to overlapping features or class imbalance in the dataset.

\subsection{Multiclass ROC and Area Under Curves [AUCs]}

The Multiclass Receiver Operating Characteristic (ROC) curve is used to measure the performance of classification models identifying multiple classes. It shows the relationship between the true positive rate and the false positive rate. The ROC curve is used to compute the AUC score. The value of the AUC score ranges from 0 to 1. The higher the AUC score, the better the model. 

Figure \ref{fig:Multiclass ROC} illustrates the separability of classes across all possible thresholds, highlighting the model's effectiveness in classifying each class. For multiclass classification, two common approaches to compare classes are One-vs-Rest (OvR) and One-vs-One (OvO). In this study, we employed the OvR strategy for all 38 classes. In the OvR approach, a separate ROC curve is generated for each class by treating that class as the positive class and the remaining classes as the negative class. The overall multiclass ROC curve is then derived by averaging these individual ROC curves. 


The OvR approach is simpler and provides insights into the performance of the classifier for each individual class. However, it can be sensitive to class imbalance. Multiclass ROC analysis allows us to assess the trade-off between TPR and FPR. It identifies the optimal operating point for the classifier based on the specific requirements of the application. 

\subsection{t-SNE Visualization}
t-SNE (t-distributed Stochastic Neighbour Embedding) is an unsupervised nonlinear dimensionality reduction technique for data exploration and visualizing high-dimensional data. It focuses on preserving the pairwise similarities between data points in a lower-dimensional space. It allows us to understand more about underlying patterns and relationships in the data by providing a
feel and intuition on how data is arranged in higher dimensions. t-SNE is often used to visualize complex datasets into two and three dimensions. It is
concerned with preserving small pairwise distances. 

\begin{figure*}[!t]
\centerline{\includegraphics[width=0.8\textwidth]{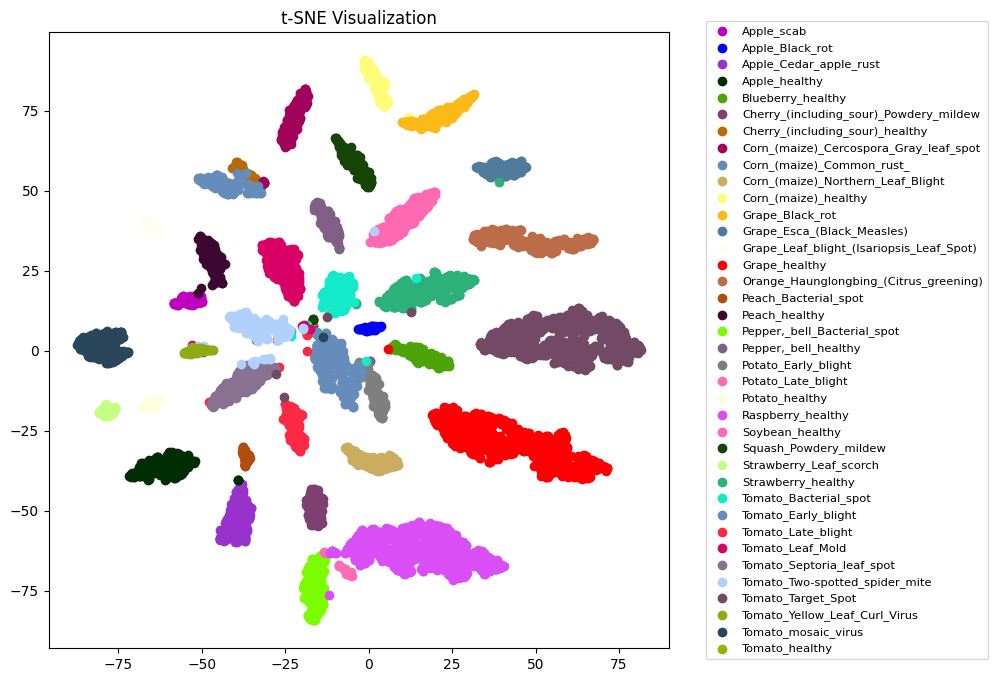}}
\caption{ t-SNE Visualization of all 38 classes of the PlantVillage dataset.}
\label{fig:tSNE Visualization}
\end{figure*}

Non-linear dimensionality reduction means that the algorithm allows us to separate data that cannot be separated by a straight line. It models a point being selected as a neighbor of another point in both higher and lower dimensions. t-SNE starts by calculating a pairwise similarity between all data points in the high-dimensional space using a Gaussian kernel. The similarity between two points is proportional to their Euclidean distance points that are close together have a higher probability of being picked as neighbors than points that are far apart. This similarity is represented by a conditional probability distribution in the high-dimensional space. The goal of t-SNE is to find a low-dimensional embedding where the similarities between points are preserved as much as possible. The algorithm tries to map higher dimensional data points onto lower dimensional space while preserving
the pairwise similarities. It uses gradient descent to minimize the divergence. The low-dimensional embedding is optimized to a stable state. The optimization process allows the creation of clusters and sub-clusters of similar points in the low-dimensional space that are visualized to understand the structure and relationship in the high-dimensional data.

\section{Conclusion and Future Scope}\label{sec5}
This paper implements the Group Equivariance Convolution techniques on the plant village dataset in order to detect if the leaves of the plants are diseased or healthy. We test the idea of exploiting symmetry and optimizing weight sharing to increase the efficacy of deep learning models. Our primary aim has been building an image classification model that can be implemented anywhere easily without much constraints. Reducing the number of parameters than the existing state-of-the-art techniques while maintaining accuracy is one of the numerous ways to satisfy that criterion. Constructed models have been modified using fine-tune. Results obtained using the proposed model are differentiated using various comparison metrics against existing methodologies. Better accuracy is achieved through the usage of ideal layer combination along with the best architectural design of the model. \par Future explorations in this domain will extend this work to solve different aspects by adding alternative components of the leaf features. One propitious course of action will be its integration with Real-time Image Processing attribute to tackle the problem of latency and bandwidth. Furthermore, One-shot and Few-shot learning procedures are likely to play significant roles to study the behavior of the diseases attacking different species of crops and further widen the application area. Following that, the optimal fit for the development of a novel model in the range of this dataset can be detected and widely-spread for mass deployment in everyday life.

\section{Declarations}

\textbf{Data Availability:} Code available at \href{https://www.kaggle.com/code/drthuganomics/weightout-plantdiseaseclassifier}{Kaggle Notebook}.

\textbf{Acknowledgement:} Part of this work was conducted in the Computer Vision and Pattern Recognition Lab, Dept. of ETCE, Jadavpur University, Kolkata.

\balance

\bibliography{reference}   
\bibliographystyle{ieeetr}   


\balance

\end{document}